\documentclass[11pt]{article}
\usepackage{fullpage}
\usepackage{listings}
\usepackage[usenames,dvipsnames,svgnames,table]{xcolor}
\usepackage{enumitem}
\usepackage{graphicx}
\usepackage{fancybox}
\usepackage{comment}
\usepackage{xcolor}
\usepackage{nameref}
\usepackage{bbm}

\definecolor{ForestGreen}{rgb}{0.1333,0.5451,0.1333}
\definecolor{DarkRed}{rgb}{0.8,0,0}
\definecolor{Red}{rgb}{1,0,0}
\usepackage[colorlinks,citecolor=blue,linkcolor=blue,urlcolor=red,pagebackref]{hyperref}
\usepackage[nottoc,numbib]{tocbibind}

\usepackage{amsmath}
\usepackage{amssymb}
\usepackage{amsthm}

\usepackage[ruled, noend, linesnumbered]{algorithm2e}

\usepackage{subfig}

\usepackage{thm-restate}

\usepackage{float}
\usepackage[capitalise,nameinlink]{cleveref}

\newtheorem{theorem}{Theorem}[section]

\newtheorem{corollary}[theorem]{Corollary}
\newtheorem{lemma}[theorem]{Lemma}
\newtheorem{observation}[theorem]{Observation}
\newtheorem{proposition}[theorem]{Proposition}
\newtheorem{claim}[theorem]{Claim}

\newtheorem{definition}[theorem]{Definition}

\newtheorem*{theorem*}{Theorem}
\newtheorem*{hypothesis*}{Hypothesis}
\newtheorem*{corollary*}{Corollary}
\newtheorem*{conjecture*}{Conjecture}
\newtheorem*{lemma*}{Lemma}
\newtheorem*{thm*}{Theorem}
\newtheorem*{prop*}{Proposition}
\newtheorem*{obs*}{Observation}
\newtheorem*{definition*}{Definition}

\newtheorem*{remark*}{Remark}
\newtheorem*{rec*}{Recommendation}

\newenvironment{fminipage}%
  {\begin{Sbox}\begin{minipage}}%
  {\end{minipage}\end{Sbox}\fbox{\TheSbox}}

\def\defeq{\stackrel{\mathrm{def}}{=}}

\def\floor#1{\left\lfloor #1 \right\rfloor}

\def\calF{\mathcal{F}}

\def\calS{\mathcal{S}}

\def\calN{\mathcal{N}}
\def\calM{\mathcal{M}}
\def\calV{\mathcal{V}}
\def\calP{\mathcal{P}}
\def\calA{\mathcal{A}}

\newcommand\R{\mathbb{R}}
\newcommand\N{\mathbb{N}}

\newcommand{\E}[1]{\mathop{{}\mathbb{E}}\left[#1\right]}

\newcommand\Z{\mathbb{Z}}

\renewcommand{\forall}{\mathrm{\text{ for all }}}

\renewcommand{\tilde}{\widetilde}

\DeclareFontFamily{U}{mathb}{\hyphenchar\font45}
\DeclareFontShape{U}{mathb}{m}{n}{<5> <6> <7> <8> <9> <10> gen * mathb
<10.95> mathb10 <12> <14.4> <17.28> <20.74> <24.88> mathb12}{}
\DeclareSymbolFont{mathb}{U}{mathb}{m}{n}
\DeclareMathSymbol{\rcirclearrow}{\mathbin}{mathb}{'367}

\newif\ifrandom
\randomtrue

\newcommand{\todolater}[1]{}

\newcommand{\cA}{\calA}

\newcommand{\cF}{\mathcal{F}}

\DeclareUnicodeCharacter{2113}{$\ell$}

\DeclareMathOperator{\Unif}{Unif}

\usepackage{appendix}
\usepackage{thm-restate}

\DeclareMathOperator{\score}{score}

\title{Algorithmic Thinking Theory}
\date{}
\author{MohammadHossein Bateni \\
Google\\
\texttt{bateni@google.com}
\and
Vincent Cohen-Addad \\
Google\\
\texttt{cohenaddad@google.com}
\and
Yuzhou Gu \\
NYU \\
\texttt{yuzhougu@nyu.edu}
\and
Silvio Lattanzi \\
Google\\
\texttt{silviol@google.com}
\and
Simon Meierhans\thanks{Simon Meierhans is supported by a Google PhD Fellowship.}\\
ETH Zurich \\
\texttt{mesimon@inf.ethz.ch}
\and
Christopher Mohri\thanks{Work done while at Google.} \\
Stanford, Google \\
\texttt{xmohri@stanford.edu}
}

\begin{document}

\maketitle

\begin{abstract}
  Large language models (LLMs) have proven to be highly effective for solving complex reasoning tasks. Surprisingly, their capabilities can often be improved by iterating on previously generated solutions. In this context, a reasoning plan for generating and combining a set of solutions can be thought of as an algorithm for reasoning using a probabilistic oracle.

  We introduce a theoretical framework for analyzing such reasoning algorithms. This framework formalizes the principles underlying popular techniques for iterative improvement and answer aggregation, providing a foundation for designing a new generation of more powerful reasoning methods.

  Unlike approaches for understanding models that rely on architectural specifics, our model is grounded in experimental evidence. As a result, it offers a general perspective that may extend to a wide range of current and future reasoning oracles.
\end{abstract}

\newpage

\section{Introduction}

The rapid advancement of Large Language Models (LLMs) has marked a paradigm shift in artificial intelligence, with models demonstrating superhuman performance on a wide array of language benchmarks and impressive capabilities on complex reasoning tasks \cite{ahn-etal-2024-large, yan-etal-2025-survey}. Initial challenges, such as grade-school mathematics (GSM8K) and standard competition math (MATH dataset), have largely been surmounted, pushing the frontier of AI reasoning toward ``grand challenge'' problems, such as those found in the International Mathematical Olympiad (IMO).

These problems, renowned for their demand for deep insight, creativity, and rigorous proof, expose a fascinating weakness in modern LLMs. While a model's performance on a single attempt (termed $pass@1$) may be very low, its ability to produce a correct answer within $k$ attempts ($pass@k$) can be significantly higher. This $pass@1$ versus $pass@k$ gap, especially pronounced when sampling with high temperature to produce diverse outputs, suggests that models possess a vast, latent capability that is not accessible in a single, high-confidence generation.

Interestingly, to recover the full power of the model it is not sufficient to simply use multiple attempts. In fact, even the $pass@k$ metric fails to capture the full story. On the most difficult problems, simply sampling $k$ times and selecting the best answer (e.g., ``best-of-32'') still yields poor results. For instance, Huang and Yang (2025) report that a best-of-32 baseline on the IMO 2025 problems achieved an accuracy of only 31.6--38.1\% for leading models \cite{huang2025imo}. This paradox lies at the heart of our work: the latent capability of LLMs is not merely a matter of \emph{selection} (finding one correct needle in a haystack of $k$ attempts), but one of \emph{synthesis}.
While established techniques like Self-Consistency~\cite{wang23} leverage majority voting to improve reliability, they remain bound by the quality of individual samples.
The ``deep'' reasoning ability of the model appears to be distributed across multiple, diverse, and often individually flawed chains of thought \cite{wei22}.

This shift highlights a growing focus on scaling inference-time compute---leveraging more computational resources \emph{after} training to boost performance on hard tasks.
These algorithms effectively simulate a ``System 2'' reasoning process atop the base model's ``System 1'' generations~\cite{Kahn11}, iteratively refining and synthesizing information to reach conclusions that are unreachable in a single forward pass.
While empirical results show this works, we lack a formal theory governing the efficient allocation of this test-time budget. Existing theoretical work often focuses on the expressivity of standard Transformer forward passes, leaving the dynamics of these iterative, multi-call reasoning systems unexplored.

While the theoretical understanding of LLMs reasoning is still limited, the hypothesis that the ability to synthesize is a key ingredient in LLMs success is strongly backed by recent empirical breakthroughs.
Work such as \emph{Reflexion}~\cite{shinn23} demonstrated the value of verbal reinforcement for iterative improvement, while Huang and Yang (2025) showed that a similar model-agnostic, multi-stage ``verification and refinement pipeline'' can achieve a stunning 85.7\% accuracy (5 out of 6 problems) on the same IMO 2025 dataset \cite{huang2025imo}. Explicitly showing that a complex reasoning procedure---which iteratively generates, improves, verifies, and refines a solution---is able to harness the model's latent abilities to achieve a result far beyond what simple sampling could.
Earlier structured approaches like \emph{Tree of Thoughts}~\cite{Yao23} attempted to harness latent reasoning by framing it as a search over intermediate steps. While effective for moderate complexity, grand challenge problems have required even more involved, iterative pipelines.

A similar, complementary approach is proposed by Venkatraman et al. (2025) with their ``Recursive Self-Aggregation'' (RSA) algorithm \cite{Ven25}. Inspired by evolutionary methods, RSA maintains a \emph{population} of candidate solutions. In each step, it refines this population by ``aggregating'' subsets of solutions, prompting the model to combine their useful ideas and produce a new, improved generation. RSA is shown to ``enable bootstrapping from partially correct intermediate steps within different chains of thought'' \cite{Ven25}. Again, this demonstrates a mechanism for \emph{synthesis}, not selection.

These empirical successes present a clear theoretical gap. We have powerful, practical examples of ``reasoning algorithms''---like the IMO pipeline or RSA---that unlock performance far exceeding a model's baseline. Yet, we lack a formal theory to understand \emph{why} they work and \emph{how} to design them. What properties of an LLM (as a ``reasoning oracle'') make aggregation and refinement effective? How do we model the trade-off between parallel branching (like RSA) and sequential depth (like the IMO pipeline)?

This paper initiates the theoretical study of the ``algorithmic thinking'' procedures. We formalize the components of these complex reasoning processes. We introduce the concept of a \textbf{reasoning oracle}, $\mathcal{A}$, which takes a \emph{context} of previously generated solutions $C \subseteq \mathcal{S}$ and produces a new solution $s \in \mathcal{S}$. The probability of success is governed by a \textbf{transfer function}, $\mathcal{F}$, which models the quality of the output based on the quality of the solutions in the context.

This formalism allows us to precisely captures the empirical techniques described so far. For example, the classic baseline $pass@1$ performance is the success probability with an empty context, $\mathcal{F}(\emptyset)$, the $pass@1$ vs. $pass@k$ gap is captured by  $\mathcal{F}(C) > \mathcal{F}(\emptyset)$ when $C$ contains $k$ solutions generated by $\mathcal{F}(\emptyset)$.

Finally, more complex techniques like Recursive Self-Aggregation \cite{Ven25} can be described by an algorithm that iteratively applies $\mathcal{F}$ to contexts $C$ with $k > 1$ solutions.

Within this framework, we can now move from empirical observation to theoretical analysis. In particular, we first formalize the strategies seen in practice by defining the \texttt{Branching Algorithm} (\cref{alg:branching}), which models tree-like synthesis, and the more efficient \texttt{Genetic Algorithm} (\cref{alg:genetic}), which re-uses solutions from a previous population, directly connecting to the evolutionary approach of RSA \cite{Ven25}. We also study the \texttt{Random Sampling Algorithm} (\cref{alg:random_sampling}), a simple algorithm that reuses a random subset of previous solutions in each step.

Our central analysis hinges on a key property of the oracle: \emph{monotonicity} (\Cref{defn:monotone}), the formal assumption that a ``better'' context of solutions leads to a better output. We then introduce and study the class of \texttt{Decaying Models} (\cref{def:decaying_model}), establishing results on algorithm behavior on such models. Then we will consider several special cases, including the \texttt{Uniform Model} (\cref{def:uniform_model}) and two models where the oracle performance degrades as context size increases (\cref{def:exp_decay,def:poly_decay}).
The former is the simplest non-trivial model under our framework, which already shows interesting behavior; the latter is closer to real-world behavior, capturing the experimental observation that performance can degrade as context size becomes too large.

\subsection{Our Results}
In \cref{sec:succ_prob}, we characterize the \emph{limits} of these algorithms and show that they obtain the maximum achievable success probability for Decaying Models (\Cref{prop:binary_opt}). We also analyze the \emph{efficiency} of the algorithm by establishing the convergence rate of the algorithm to optimal success probability (\cref{sec:convergence}).
Then we make further discussions on the Uniform Model and specific Decaying Models in \cref{sec:uniform} and \cref{sec:decay}, respectively.
Our goal is to move beyond empirical successes and develop a rigorous theory for designing and analyzing the next generation of reasoning algorithms.

In \Cref{sec:experiments}, we provide experimental evidence to ground our modeling choices.

\subsection{Technical Overview}

We model a reasoning oracle as a probabilistic function $\mathcal{A}$ that samples from a solution space $\calS$. The function can either be called with empty context $C = \emptyset$, or it can be provided with some previous solutions sampled from the solution space. If it is called with non-empty context, the properties of the provided solutions alter the distribution from which the solutions are sampled.

In this article, we focus on the following simple setting. Each solution is either correct or wrong, and the distribution from which $\mathcal{A}$ samples only depends on a) if there is a correct solution in the provided context and b) the number of solutions in the context. This allows us to capture two important properties observed in practice. Providing correct solutions helps, and hiding a correct solution alongside a lot of wrong solutions decreases its usefulness.

\paragraph{Models. } Our framework leaves significant flexibility for modeling. We focus our initial study on the class of Decaying Models, where the success probability depends only on whether there is a correct solution in the context $C$, and the context size $|C|$.

We derive the maximum possible success probability one can boost to for such models, and derive the convergence rate to this success probability for the algorithms we consider.

Then we study the following two special cases of the Decaying Model.
\begin{enumerate}
  \item The simplest possible model succeeds with probability $q > p$ whenever there is a correct solution in the context $C$ and $|C| \leq k$ (See \Cref{sec:uniform}).
  \item To better capture the experiments, we replace the fixed probability $q$ with some function $f$ that decays with the context length $|C|$. In this article, we study exponential and polynomial decay functions (See \Cref{sec:decay}).
\end{enumerate}

We give efficient algorithms for boosting the success probabilities in terms of the number of oracle calls. In many regimes, these are small multiples of $1/p$, the number of oracle calls needed for obtaining a single correct solution among all the solutions contained so far.

It may be unreasonable to hope to solve a very difficult problem without providing any context. We therefore propose various ideas for generalizing our model to partial solutions in \Cref{sec:conclusion}, and hope that such a model can be both motivated by experiments and analyzed in the future.

\subsection{Other Theoretical Models for LLMs}

Early theoretical work focused on the expressivity of the architecture, proving that Transformers are universal approximators of sequence-to-sequence functions, meaning they can theoretically model any such function as then number of parameters approaches infinity \cite{yun2020are}. Since then, there have been various proposals for understanding the experimentally observed benefit of providing context.

A leading theoretical approach explains in-context learning (ICL) as an algorithmic mechanism called ``Induction Heads.'' These are 2-layer circuits that allow the model to copy and complete patterns from the context~\cite{olsson2022incontext}.
Another major theoretical angle formalizes ICL as an algorithm learning problem, where the Transformer implicitly constructs a hypothesis function at inference time~\cite{li2023transformers}. Theoretical work has shown that Transformers trained by gradient flow can provably learn classes of linear models in-context~\cite{garg2022what}.

Recent work has moved beyond showing CoT helps to proving it is theoretically required for certain classes of problems.
A recent, pivotal study proved rigorous separation results, showing that bounded-depth Transformers cannot directly solve basic arithmetic or linear equations unless their size grows super-polynomially. However, with CoT (autoregressive intermediate steps), they can solve these, effectively allowing them to handle problems up to P-complete (like dynamic programming)~\cite{feng2024revealing}.
Newer theoretical frameworks have derived scaling laws for an ``optimal'' CoT length, proving that while CoT is necessary, excessively long chains are susceptible to noise and error accumulation (``overthinking'')~\cite{zhang2025more}. In our model, we explain this behavior as the context becoming increasingly correlated and therefore providing less value.

Another line of complexity theoretic work established that both Transformers and SSMs generally fall into the TC0 computational complexity class~\cite{sarrof2024expressive}.
Despite being in the same broad class, they have distinct theoretical strengths. Transformers are provably better at state tracking and copying from long history (due to their ability to attend to any previous state), while SSMs excels at different types of structured state tracking without needing a full history cache.

Finally, recent work models grokking as a phase transition where the network shifts from a high-complexity ``memorization'' phase to a low-complexity ``generalization'' phase, effectively compressing the data~\cite{kumar2024complexity}.
More recent studies have begun linking this strictly to internal mechanisms, such as the evolution of routing pathways in Mixture of Experts (MoE) models during pretraining~\cite{grokking2025moe}.

\section{Preliminaries} \label{sec:prelim}

\paragraph{Notations.} We let $\log(\cdot)$ refer to the natural logarithm and $[n] = \{1, \ldots, n\}$. $\N = \Z_{\ge 0}$ denotes the set of non-negative integers.

\paragraph{Question.} Throughout the article, we fix a question $Q$ which we seek to answer.

\paragraph{Solution space and quality.} Let $\calS$ be the space of possible solutions of question $Q$ and $\calV \subseteq \R$ denote the space of possible quality scores. We associate with $\calS$ a scoring function $\score: \calS \to \calV$. For every solution $s \in \calS$, the value $\score(s)$ assesses to what extent solution $s$ answers question $Q$. A reasoning algorithm does not have access to the function $\score$. In this article, we will focus on $\calV = \{0,1\}$, i.e., the scoring function where every solution has a binary score.

\paragraph{Reasoning oracle.} To produce answers to question $Q$, we will exclusively use an oracle $\calA$. This oracle is a randomized algorithm that given context $C \subseteq \calS$ samples a solution $s \in \calS$. The distribution from which the oracle samples depends only on the quality of the solutions in the context $C$.

We assume that distribution of the score of the output of the oracle depends only on the multiset of scores in the context $C$, in the following way. Let $\calM(\calV)$ denote the set of all finite multisets with elements in $\calV$, and $\calP(\calV)$ denote the set of distributions over $\calV$.
We assume there exists a transfer function $\cF: \calM(\calV) \to \calP(\calV)$ such that the distribution of the score of the output of the oracle on context $C$ is given by $\cF(\{\score(s): s\in C\}_{\text{multiset}})$.
In particular, $\cF(\emptyset)$ denotes the distribution of scores when the oracle is called with an empty context.

\paragraph{Reasoning algorithm.} An $(\calA, n)$-reasoning algorithm is a procedure that calls the oracle $\calA$ at most $n$ times to output a solution $s$.
The goal is to maximize expected score $\E{\score(s)}$ of the final output $s$.
If $\calV = \{0, 1\}$, then we say a reasoning algorithm succeeds if $\score(s) = 1$.

In general, a reasoning algorithm has the form $(S_1,\ldots,S_n)$, where $S_k \subseteq [k-1]$ indicates the previous outputs used as context for the $k$-th oracle call.
That is, for $k=1,\ldots,n$, the algorithm calls the oracle $\calA$ with context $\{s_i: i \in S_k\}_{\text{multiset}}$ to obtain $s_k$.

We can represent a reasoning algorithm as a DAG with $n$ nodes, where each node corresponds to an oracle call, and there is a directed edge from node $i$ to node $j$ if $i \in S_j$. The $n$-th node corresponds to the final output of the algorithm.
We define the \emph{depth} of an algorithm as the length (number of edges) of the longest paths in its DAG representation.
The depth captures the inference time of the reasoning algorithm, assuming infinite parallelism.

\paragraph{Answer population.} If the image of the $\score$ function is $\{0, 1\}$, then we call a subset $S \subset \calS$ a $p$-correct answer population if at least a $p$ fraction of the elements in $S$ have score $1$. For the rest of the article, we let $\calV = \{0,1\}$.\footnote{Some of the general results presented in \Cref{sec:succ_prob} apply to general sets $\calV$.}

\paragraph{Decaying Model.}
Most of our results focus on a class of models called the decaying model, defined as follows.
\begin{restatable}[Decaying Model]{definition}{decayingModel}\label[definition]{def:decaying_model}
  Let $\calV = \{0,1\}$. Let $f: \N \to [0, 1]$ and $g: \N \to [0, 1]$ be two functions satisfying $f(k) \ge g(k)$ for all $k\in \N$.
  The decaying model $\calA^{(f,g)}$ is defined as the oracle satisfying for any context $C \subseteq \calS$:
  \begin{align*}
    \Pr[\score(\calA^{(f,g)}(C)) = 1] =
    \begin{cases}
      f(|C|), & \text{if } \sum_{s \in C}\score(s) \geq 1, \\
      g(|C|), & \text{if } \sum_{s \in C}\score(s) = 0.
    \end{cases}
  \end{align*}
\end{restatable}

We note that the functions $f$ and $g$ are typically monotonically decreasing, which motivates the nomenclature. For this decaying model, we are able to give a precise characterization of the optimal success probability achievable by reasoning algorithms (see \cref{sec:succ_prob}).

Next we will introduce several special cases of the decaying model.

\paragraph{Uniform Model.}
The simplest model is the following.
\begin{restatable}[Uniform Model]{definition}{uniformModel}
\label[definition]{def:uniform_model}
  Let $0\le p \le q\le 1$ and integer $k\ge 1$.
  For the uniform model $\calA^{(p,q,k)}_u$, each call has $\score(\calA^{(p,q,k)}_u(C)) = 1$ independently with probability
  \begin{align*}
    p + (q - p) \mathbbm{1}\left\{\sum_{s \in C}\score(s) \geq 1 \text{ and } |C| \leq k\right\}
  \end{align*}
  and $\score(\calA^{(p,q,k)}_u) = 0$ otherwise.
\end{restatable}
The uniform model is the special case of the decaying model where $f(k) = q$ and $g(k) = p$ for all $k\in \N$.

If we remove the constraint $|C|\le k$ in the definition, then the optimal strategy is to always use all previous answers as context. We therefore limit the context size $|C|$ to be at most $k$ in the definition of the uniform model.

\paragraph{Exponentially and Polynomially Decaying Models.}
When implementing the oracle with a large language model, we experimentally observe a slow decay in model accuracy as the context size increases. We therefore study models that capture this decay more explicitly.

We are particularly interested in two natural decay functions, exponential decay and polynomial decay.

\begin{restatable}[Exponential Decay]{definition}{expDec}
\label[definition]{def:exp_decay}
  Let $0<p\le q\le 1$.
  The exponential decaying model is the decaying model $\calA^{(f,g)}$ where $f(k) = q^{k-1}$ and $g(k) = p\mathbbm{1}_{\{k = 0\}}$.
\end{restatable}

\begin{restatable}[Polynomial Decay]{definition}{polyDec} \label[definition]{def:poly_decay}
  Let $0<p\le q\le 1$.
  The polynomial decaying model is the decaying model $\calA^{(f,g)}$ where $f(k) = \frac 1{k^q}$ and $g(k) = p\mathbbm{1}_{\{k = 0\}}$.
\end{restatable}

\section{Reasoning Algorithms} \label{sec:algorithm}

In this section, we introduce the  reasoning algorithms we study in this article:
\begin{enumerate}
	\item \underline{Branching Algorithm (\Cref{alg:branching}):} A simple algorithm that merges independently generated solutions in a tree-like fashion. It achieves optimal success probability for a broad class of oracles.
	\item \underline{Genetic Algorithm (\Cref{alg:genetic}):} A more efficient version of the branching algorithm that re-uses solutions to reduce the number of oracle calls. It approaches the branching algorithm as the population sizes grow.
	\item \underline{Random Sampling Algorithm (\Cref{alg:random_sampling}):} An algorithm that generates new solutions by randomly sampling from all previously generated solutions. It also achieves the optimal success probability, and has better convergence rate in certain settings.
\end{enumerate}

\paragraph{Branching algorithm. } The branching algorithm assembles a solution by initially generating a set of $pass@1$ solutions without providing any context. We call such solutions level $0$ solutions. Then, it builds a tree obtaining a set of level $i$ solutions via combining groups of $k_i$ solutions from level $i - 1$ until it creates a single level $L$ solution. See \Cref{alg:branching} for pseudo-code of this algorithm.

\begin{algorithm}[h]\caption{\textsc{BranchingAlgorithm}$(\calA, L, (k_1, \ldots, k_L))$)} \label{alg:branching}
	\If{$L = 0$}{
		\Return $\calA(\emptyset)$
	}
	$C \gets \emptyset$ \\
	\For{$j \in 1, \ldots, k_L$}{
		$C \gets C \cup \{\textsc{BranchingAlgorithm}(\calA, L - 1, (k_1, \ldots, k_{L - 1}))\}$
	}
	\Return $\calA(C)$
\end{algorithm}

From this construction, we observe that all solutions generated by this algorithm are completely independent, and that therefore the solutions at level $i$ also have to have the same guarantees. We conclude that all the solutions at level $i$ have the same success probability, and that any reasonable algorithm increases the success probability of the solutions over time (See \Cref{sec:succ_prob}).

\paragraph{Genetic Algorithm.} While the branching algorithm achieves the best possible success probability for a broad class of oracles (\Cref{prop:binary_opt}), it exponentially grows the number of solutions with the depth $L$. This quickly blows up the necessary computational resources. Therefore, genetic algorithms are often used in practice \cite{Ven25}. These essentially follow the template from \Cref{alg:branching}, but re-use solutions from level $i - 1$ multiple times to avoid branching as much. Concretely, they have a fixed population size for layer $i$, and each solution at layer $i > 0$ is created by sampling $k_i$ solutions from layer $i - 1$ uniformly at random with replacement. See \Cref{alg:genetic} for pseudo code of the algorithm and note that the population sizes $(s_1, \ldots, s_L)$ are an additional input. We observe that this algorithm approaches the branching algorithm as the population sizes go towards infinity. %

\begin{observation}\label{obs:genetic_branching}
	The genetic algorithm (\Cref{alg:genetic}) approaches the branching algorithm (\Cref{alg:branching}) when $s_1, \ldots, s_L \to \infty$ and $\frac{s_i}{s_{i + 1}} \to \infty$ for $i = 1, \ldots L - 1$.
\end{observation}

\begin{algorithm}[h]
\caption{\textsc{GeneticAlgorithm}$(\calA, L, (s_1, \ldots, s_L), (k_1, \ldots, k_L)$)} \label{alg:genetic}
$S_0 \gets \emptyset$ \\
\For{$j = 1, \ldots, s_1$}{
	$S_0 \gets S_0 \cup \{\calA(\emptyset)\}$
}
\For{$i = 1, \ldots, L - 1$}{
	$S_i \gets \emptyset$ \\
	\For{$j = 1, \ldots, s_{i + 1}$}{
		Let $C$ consist of $k_{i}$ i.i.d.~uniformly sampled items from $S_{i - 1}$ \\
		$S_i \gets S_i \cup \{\calA(C)\}$
	}
}
Let $C$ consist of $k_{L}$ i.i.d.~uniformly sampled items from $S_{L - 1}$ \\
\Return $\calA(C)$
\end{algorithm}

In the following, we show that any branching algorithm can be turned into a genetic algorithm given access to a slightly stronger oracle in the case where every solution is either wrong or correct (i.e., $\calV = \{0, 1\})$.

Recall the following standard Chernoff bound.
\begin{theorem}[Chernoff bound, see e.g. \cite{MitzenmacherUpfal05}] \label{thm:chernoff_lower}
	Suppose that $X_1, \ldots, X_n$ are i.i.d random variables taking values in $\{0, 1\}$. Let $X = \sum_{i = 1}^n X_i$ and $\mu = \E{X}$. Then
	\begin{align*}
		\Pr[X \leq (1 - \epsilon) \mu] \leq e^{-\epsilon^2 \mu/2}.
	\end{align*}
\end{theorem}

\begin{definition}
	We say oracle $\calA'(\cdot)$ $\alpha$-dominates oracle $\calA(\cdot)$ if the success probability of $\calA'(\cdot)$ is at least $\alpha$ times the success probability of oracle $\calA$ for some $\alpha$ for all inputs to $\mathcal{A}$.
\end{definition}

\begin{lemma}\label{lem:branching_to_genetic}
	Suppose an oracle $\calA(\cdot)$, a parameter $L$ and a vector $(k_1, \ldots, k_L)$ are given such that $\textsc{BranchingAlgorithm}(\calA, i, (k_1, \ldots, k_i))$ returns a solution $s$ that is correct with probability $p_i$ for each $1\leq i\leq L$. Then, given an oracle $\calA'$ that $\frac{1}{1-\epsilon}$ dominates the oracle $\calA$,  $\textsc{GeneticAlgorithm} (\calA', L, (s_1, \ldots, s_L), (k_1, \ldots, k_L)$ returns a solution $s'$ that is correct with probability at least $p_L$, provided that $s_i \geq -2\log(\epsilon/L)/(p_{i - 1}\epsilon^2)$ for each $1\leq i\leq L$.
\end{lemma}
\begin{proof}
	We prove by induction that for each $0\leq i\leq L$, the set $S_i$ is a $p_{i - 1}$-correct answer population with probability at least $1 - i \cdot \epsilon/L$.

	For $S_0$, we sample directly from an oracle which guarantees that the solutions are correct with probability $\frac{1}{1 - \epsilon} p_0$. Therefore, $s_1 = -2\log(\epsilon/L)/(p_{0}\epsilon^2)$ samples suffices with probability $1 - \epsilon/L$ by \Cref{thm:chernoff_lower}. This concludes the base case.

	Then, in the step case, we have that if $S_{i - 1}$ is a $p_{i - 2}$-correct answer population with probability  $1 - (i - 1)\epsilon/L$ by induction. We assume that it is in fact a $p_{i - 2}$-correct answer population, and then obtain that $S_{i}$ is a $p_{i - 1}$-correct answer population with probability $1 - \epsilon/L$ by \Cref{thm:chernoff_lower}. The step case follows by union bound.

	From this, we obtain that $S_L$ is a $p_{L - 1}$-correct answer population with probability $1 - \epsilon$. If it is, then the output is correct with probability $\frac{1}{1 - \epsilon}p_L$. The lemma follows.
\end{proof}

\paragraph{Random Sampling Algorithm.}
The genetic algorithm samples solutions from the previous layer only. An alternative is to sample solutions from all previously generated solutions. This is the idea behind the random sampling algorithm, whose pseudo-code is given in \Cref{alg:random_sampling}.
\begin{algorithm}[h]
\caption{\textsc{RandomSamplingAlgorithm}$(\calA, n, k$)} \label{alg:random_sampling}
$s_1 \gets \calA(\emptyset)$ \\
\For{$j = 2, \ldots, n$}{
	Let $C$ consist of $k$ i.i.d.~uniformly sampled items from $\{s_1, \ldots, s_{j-1}\}$ \\
	$s_j \gets \calA(C)$
}
\Return $s_n$
\end{algorithm}

We will see in \cref{sec:succ_prob} that this algorithm also achieves optimal success probability for decaying models.

\section{Optimal Success Probability} \label{sec:succ_prob}
In this section, we study the optimal success probability achievable by reasoning algorithms (without restrictions on the number of calls). We show that for decaying models, the algorithms introduced in \Cref{sec:algorithm} achieve the optimal success probability as the number of calls goes to $\infty$.

\subsection{Monotonicity}
In this section we establish some general properties of reasoning oracles and algorithms.
We study two versions of monotonicity, and show that they have interesting consequences to the structure of an optimal algorithm.

\begin{definition}[Monotonicity] \label{defn:monotone}
  An oracle $\calA$ is called
  \begin{itemize}
    \item \emph{weakly monotone} if
  for any two multisets $A=\{a_1,\ldots,a_k\}$, $B=\{b_1,\ldots,b_k\}$ satisfying $a_i\le b_i$ for all $i\in[k]$, we have $\calF(A)\preceq \calF(B)$, where $\preceq$ denotes (first order) stochastic dominance (i.e., there is a monotone coupling between $\calF(A)$ and $\calF(B)$);
    \item \emph{strongly monotone} if it is weakly monotone and for any two multisets $A$ and $B$ satisfying $A\subseteq B$ (as multisets), we have $\calF(A)\preceq \calF(B)$.
  \end{itemize}
\end{definition}

The following result shows that weak monotonicity implies an FKG inequality for the joint distribution of scores of solutions \cite{FKG}.

\begin{lemma} \label{lem:monotone-weak}
  If an oracle $\calA$ is weakly monotone,
  then for any $(\calA,n)$-reasoning algorithm $(S_1,\ldots,S_n)$ and any functions $f,g:\calV^n\to \R$ that are both non-decreasing (or both non-increasing), we have
  \begin{align*}
    \E{fg} \ge \E f\E g.
  \end{align*}
\end{lemma}
\begin{proof}
  Let $v_i = \score(s_i)$ for $i\in [n]$.
  Let $(X_1,\ldots,X_n)$ be i.i.d. samples from $\Unif([0, 1])$.
  For every $j \in [n]$ and every possible combination of values $(v_i: i\in S_j)$, we can couple $\Unif([0, 1])$ with $P_{v_j | (v_i: i\in S_j)}$ in an order-preserving way.
  In this way, we can construct an order-preserving coupling between $(v_1,\ldots,v_n)$ and $(X_1,\ldots,X_n)$.
  So $f,g$ can be lifted to non-decreasing (or non-increasing) functions $\tilde{f},\tilde{g}:[0,1]^n\to \R$ such that $f(v_1,\ldots,v_n) = \tilde{f}(X_1,\ldots,X_n)$ and $g(v_1,\ldots,v_n) = \tilde{g}(X_1,\ldots,X_n)$.
  The FKG inequality \cite{FKG} implies that
  \begin{align*}
    \E{fg} = \E{\tilde{f}\tilde{g}} \ge \E{\tilde{f}}\E{\tilde{g}} = \E f\E g.
  \end{align*}
\end{proof}

Strong monotonicity, in addition to the above, implies that an optimal algorithm must use all available context size.
\begin{lemma}
\label{lem:monotone-strong}
  If an oracle $\calA$ is strongly monotone,
  then for any $(\calA,n)$-reasoning algorithm, adding more solutions to the context of any oracle call never decreases the expected score of the final output.
\end{lemma}
\begin{proof}
  For two $(\calA,n)$-reasoning algorithms operating on contexts $(S_1,\ldots,S_n)$ and $(S'_1,\ldots,S'_n)$ respectively satisfying $S_i\subseteq S'_i$, we can construct a monotone coupling step by step between the distributions of $(\score(s_1),\ldots,\score(s_n))$ and $(\score(s'_1),\ldots,\score(s'_n))$.
  So adding more solutions to the context never hurts.
\end{proof}

In particular, if there are no limits on context size, then the optimal algorithm is to take $S_i$ to contain all previously generated solutions for all $i$.

For the binary quality model considered in this article (i.e., $\calV=\{0,1\}$), the transfer function $\calF$ can be compactly represented by a function $F: \N\times \N\to [0,1]$, where $F(a,b)$ is the probability that the oracle outputs a correct solution when the context contains $a$ correct solutions and $b$ incorrect solutions. With minor abuse of notation, we call $F$ the transfer function as well.

A binary quality model is weakly monotone if and only if $F$ satisfies $F(a+1,b)\ge F(a,b+1)$ for all $a,b\ge 0$.
A binary quality model is strongly monotone if and only if $F$ satisfies $F(a+1,b)\ge F(a,b+1)\ge F(a,b)$ for all $a,b\ge 0$.
So we have the following corollary.

\begin{corollary}[Monotonicity for binary quality models] \label{cor:monotone_binary}
  Consider a binary quality model with transfer function $F$.

  If $F(a+1,b)\ge F(a,b+1)$ for all $a,b\ge 0$, then for any $(\calA,n)$-reasoning algorithm $(S_1,\ldots,S_n)$ and any functions $f,g:\{0,1\}^n\to \R$ that are both non-decreasing (or both non-increasing), we have $\E{fg} \ge \E f\E g$.

  If $F(a+1,b)\ge F(a,b+1) \ge F(a,b)$ for all $a,b\ge 0$, then for any $(\calA,n)$-reasoning algorithm, adding more solutions to the context of any oracle call never decreases the expected score of the final output.
\end{corollary}
\begin{proof}
  Directly follows from \Cref{lem:monotone-weak} and \Cref{lem:monotone-strong}.
\end{proof}

\subsection{Optimality of the branching algorithm}
In the following, we show that for decaying models, the branching algorithm achieves the maximum achievable success probability.

The following lemma discusses solutions to a class of polynomial equations which arise in this study.
\begin{lemma} \label{lem:poly_eqn}
  Let $p,q\in [0, 1]$ and $k\in \N$.
  Consider the equation $x = q - (1 - x)^k (q - p)$ and its solutions in $[0, 1]$.
  \begin{itemize}
    \item When $k=0$, the equation has a unique solution $x=p$.
    \item When $k=1$ and $(p,q)\ne (0,1)$, the equation has a unique solution $x=\frac{p}{1-q+p}$. If $k=1$ and $(p,q)=(0,1)$, every $x\in [0, 1]$ is a solution.
    \item When $k\ge 2$ and $p\ge q$, the equation has a unique solution in $[0, 1]$.
    When $k\ge 2$ and $0<p<q$, the equation has a unique solution in $[0, 1]$.
    When $k\ge 2$, $p=0$, and $k q\le 1$, the equation has a unique solution $0$ in $[0, 1]$.
    When $k\ge 2$, $p=0$, and $k q > 1$, the equation has two solutions in $[0, 1]$, one of which is $0$ and the other is in $(0, 1]$.
  \end{itemize}
\end{lemma}
\begin{proof}
  Let $f(x) = x - q + (1 - x)^k (q - p)$.
  Then $f(0) = -p \le 0$, and $f(1) = 1 - q \ge 0$.
  The $k=0,1$ cases follow directly. We therefore assume $k \ge 2$.

  Suppose $p\ge q$. Then $f$ is strictly increasing in $[0, 1]$, and there is a unique root in $[0, 1]$.

  Suppose $p<q$. Then $f$ is strictly convex in $[0, 1]$. When $p>0$, there is a unique root in $[0, 1]$.
  When $p=0$, $0$ is a root, and $f'(0)=1-k q$.
  So when $k q\le 1$, there is no root in $(0, 1]$; when $k q > 1$, there is a unique root in $(0, 1]$.
\end{proof}

A decaying model is a binary quality model whose transfer function $F$ satisfies $F(a+1,b)=F(a,b+1)$ for all $a\ge 1, b\ge 0$. We have $f(k) = F(k,0)$ and $g(k) = F(0,k)$ for all $k\in \N$.

\begin{proposition}[Optimality of branching algorithm] \label{prop:binary_opt}
  Consider a decaying model.

  Let $x_k^*$ be the largest solution in $[0, 1]$ to the equation $x=f(k)-(1-x)^k(f(k)-g(k))$ (c.f.~\cref{lem:poly_eqn}).

  Then the success probability of any $(\calA,n)$-reasoning algorithm is at most $x^*=\sup_{k\ge 0} x_k^*$.

  Furthermore, if $g(0)>0$, then for any $\epsilon>0$, a branching algorithm (\cref{alg:branching}) achieves success probability at least $x^*-\epsilon$.
\end{proposition}
\begin{proof}
  We prove lower and upper bounds separately.

  \underline{Upper bound.}
  Let $(S_1,\ldots,S_n)$ be an $(\calA,n)$-reasoning algorithm.
  Let $v_i = \score(s_i)$ for $i\in [n]$.
  We prove by induction that $\Pr[v_i=1] \le x^*$ for all $i\in [n]$.

  Fix any $m\in [n]$. Let $k_m = |S_m|$.
  By the induction hypothesis, for each $i\in S_m$, we have $\Pr[v_i=0] \ge 1-x^*$.
  By \cref{cor:monotone_binary}, $\Pr[v_i=0\forall j\in S_i] \ge (1-x^*)^{k_m}$.

  Then $\Pr[v_i] \le f(k_m) - (1 - x^*)^{k_m} (f(k_m) - g(k_m)) \le x^*$,
  where last step follows from the definition of $x^*$.

  \underline{Lower bound.}
  Choose $k$ such that $x_k^* > x^* - \epsilon$.
  Consider the branching algorithm \Cref{alg:branching} with fan-in $k_i = k$ for all $i \in 1, \ldots, L$. Since all the generated solutions at a particular level are independent, the success probability evolves according to the function whose fixed point we are searching. Then, it is easy to observe that the success probability monotonically increases across levels, and goes to $x_k^*$ in the limit.
\end{proof}

In fact, the branching algorithm achieves optimal success probability not only in the limit, but also among algorithms with a fixed depth.

\begin{proposition}[Optimality of branching algorithm for fixed-depth algorithms] \label{prop:binary_opt_fixed_depth}
  Work under the setting of \cref{prop:binary_opt}.
  Among algorithms with a fixed depth $L$, a branching algorithm with $L$ levels achieves the maximum success probability.
\end{proposition}
\begin{proof}
  Let $a_0 = g(0)$ and $a_i = \max_k \left(f(k) - (1 - a_{i-1})^k (f(k) - g(k))\right)$ for $i\ge 1$.
  Clearly, a depth-$L$ branching algorithm can achieve success probability $a_L$.
  Furthermore, using the same induction as in the proof of \cref{prop:binary_opt}, we can show that any depth-$L$ algorithm has success probability at most $a_L$.
\end{proof}

These optimality results can be understood as that for these models, independence between solutions is favored. We will see in \cref{sec:uniform} that the natural algorithm that always passes the most recently generated solutions is not optimal because of said dependencies.

\subsection{Optimality of the genetic algorithm}
Via \cref{lem:branching_to_genetic}, \cref{prop:binary_opt} implies that the genetic algorithm also achieves optimal success probability.
\begin{proposition} \label{prop:genetic_succ_prob}
  Work under the setting of \cref{prop:binary_opt}.
  Fix $k\ge 1$ and $\epsilon>0$.
  The Genetic Algorithm (\Cref{alg:genetic}) with $k$-way branching achieves success probability at least $x_k^*-\epsilon$ for large enough $L$, where $x_k^*$ is the unique solution of $x = f(k) - (f(k) - g(k))(1 - x)^k$ in $[0,1]$.
\end{proposition}
\begin{proof}
  Let $\cA$ denote the given oracle.
  Choose $\epsilon_1>0$ small enough such that the unique solution to $x(1-\epsilon_1) = f(k) - (f(k) - g(k))(1 - x)^k$ in $[0,1]$ is at least $x_k^* - \epsilon$.
  This means that there is a decaying model that is $1/(1-\epsilon_1)$ dominated by $\cA$ such that the optimal success probability is at least $x_k^* - \epsilon$.
  Then \cref{lem:branching_to_genetic} shows that there is a genetic algorithm that achieves success probability at least $x_k^* - \epsilon$.
\end{proof}

\subsection{Optimality of the random sampling algorithm}
The analysis of the random sampling algorithm relies on the theory of stochastic approximation \cite{duflo2013random}. We use the following classic result.

\begin{theorem}[{\cite[Theorem 2.2.12]{duflo2013random}}] \label{thm:stochastic_approx}
  Suppose $(X_n)_{n\ge 1}$ and $(Y_n)_{n\ge 1}$ are two sequences of random variables in $\R$ that are adapted to a filtration $(\calF_n)_{n\ge 1}$ and linked by the equation $X_{n+1}=X_n+\frac 1n Y_{n+1}$.
  Suppose
  \begin{align*}
    \E{Y_{n+1}|\calF_n} &= f(X_n),\\
    \E{(Y_{n+1}-f(X_n))^2|\calF_n} &=\Gamma(X_n).
  \end{align*}
  In addition, suppose that the following holds for a finite constant $K$:
  \begin{enumerate}
    \item $f$ is a function of class $C^2$ such that
    \begin{align*}
      f(x)^2 &\le K(1+x^2), \\
      f(x^*)&=0~\text{and for}~x\ne x^*, f(x)(x-x^*) < 0;
    \end{align*}
    \item $\Gamma$ is continuous in a neighborhood of $x^*$ and $\Gamma(x) \le K(1+x^2)$;
    \item There exists an $a>0$ such that $\sup_n \E{|Y_{n+1}-f(X_n)|^{2+a}|\calF_n} <\infty$.
  \end{enumerate}

  Then we have the following.
  \begin{enumerate}
    \item $(X_n)_{n\ge 1}$ converges to $x^*$ almost surely.
    \item Let $\tau = -f'(x^*)$ and $\Gamma^* = \Gamma(x^*)$.
    \begin{enumerate}
      \item If $\tau>1/2$, then $\sqrt{n}(X_n - x^*)$ converges in distribution to $\calN(0, \Gamma^*/(2\tau-1))$.
      \item If $\tau=1/2$, then $\sqrt{n/\log n}(X_n - x^*)$ converges in distribution to $\calN(0, \Gamma^*)$.
      \item If $0<\tau<1/2$, then $n^\tau (X_n - x^*)$ converges almost surely to a finite random variable.
    \end{enumerate}
  \end{enumerate}
\end{theorem}

\begin{proposition} \label{prop:random_sampling_succ_prob}
  Work under the setting of \cref{prop:binary_opt}.
  Fix $k\ge 1$ and $\epsilon>0$.
  If $g(k)>0$, then the Random Sampling Algorithm (\Cref{alg:random_sampling}) with context size $k$ achieves success probability at least $x_k^*-\epsilon$ for large enough $n$, where $x_k^*$ is the unique solution of $x = f(k) - (f(k) - g(k))(1 - x)^k$ in $[0,1]$.
\end{proposition}
\begin{proof}
  Define $X_t$ as the fraction of correct solutions in $\{s_1, \ldots, s_t\}$.
  Then we have
  \begin{align*}
    X_{t+1} = X_t + \frac{1}{t+1}(-X_t + Y_t)
  \end{align*}
  where $Y_t=1$ with probability $f(k) + (f(k) - g(k))(1 - X_t)^k$ and $Y_t=0$ otherwise.
  Define $h(x) = f(k) - (f(k) - g(k))(1 - x)^k - x$.
  Then we have
  \begin{align*}
    X_{t+1} = X_t + \frac{1}{t+1}(h(X_t) + Z_t)
  \end{align*}
  where $Z_t = Y_t - \mathbb{E}[Y_t|X_t]$ is a martingale difference sequence with bounded variance.

  To apply \cref{thm:stochastic_approx}, we need to verify that $h(x)(x-x_k^*)<0$ for $x\ne x_k^*$.
  Note that $h(x)$ is concave on $[0, 1]$ and $h(0)=g(k)>0$. So $h(x)>0$ for $0\le x<x_k^*$ and $h(x)<0$ for $x_k^*<x\le 1$.
  Thus the condition holds.

  By \cref{thm:stochastic_approx}, $X_t$ converges almost surely to the unique root of $h(x)=0$, which is exactly $x_k^*$.
\end{proof}

\section{Convergence Analysis} \label{sec:convergence}
In this section we analyze the convergence rate of different reasoning algorithms.
For simplicity, we focus on uniform models, but our analysis can be extended to general decaying models as well.

We fix a uniform model $\calA^{(p,q,k)}_u$ as defined in \Cref{def:uniform_model}, where $0<p<q<1$ and integer $k\ge 1$.
Let $x^*$ be the unique solution in $[0,1]$ of the equation $x = q - (q - p)(1 - x)^k$.
By \cref{prop:binary_opt}, the optimal success probability of any reasoning algorithm using the oracle $\calA^{(p,q,k)}_u$ converges to $x^*$ as the number of steps goes to infinity.

\subsection{Convergence analysis of the branching algorithm}
\begin{proposition} \label{prop:branching_convergence}
  The success probability of the branching algorithm (\Cref{alg:branching}) with $L$ levels of $k$-way branching is $x^*-\left(k\frac{q-x^*}{1-x^*} \pm o(1)\right)^{L}$.

  In terms of the number of calls $n = k^L$, the success probability is $x^*-\epsilon_n$ where
  \begin{align*}
    \lim_{n\to \infty} \frac{\log(1/\epsilon_n)}{\log n} = \frac{-\log \left(k\frac{q-x^*}{1-x^*}\right)}{\log k}.
  \end{align*}
\end{proposition}
\begin{proof}
  Define $f(x) = q - (q - p)(1 - x)^k$.
  Let $x_j$ be the success probability after $j$ levels of branching.
  We have $x_0=p$ and $x_{j+1} = f(x_j)$ for all $j\ge 0$.

  We have shown in \cref{prop:binary_opt} that $x_j \to x^*$ as $j \to \infty$.
  The convergence rate is governed by $f'(x^*) = k(q - p)(1 - x^*)^{k-1} = k\frac{q - x^*}{1 - x^*}$.
  By the proof of \cref{prop:random_sampling_succ_prob}, we have $0 < f'(x^*) < 1$.
  So $\lim_{j\to \infty} \frac{|x_{j+1}-x^*|}{|x_j - x^*|} = f'(x^*)$.
  Therefore, $|x_L - x^*| = (f'(x^*) \pm o(1))^{L}$.
\end{proof}

\subsection{Convergence analysis of the genetic algorithm}
\begin{proposition} \label{prop:genetic_convergence}
  There exists a Genetic Algorithm (\Cref{alg:genetic}) with $k$-way branching whose success probability $x^*-\epsilon_n$ (where $n$ is the number of calls to the oracle) satisfies
  \begin{align*}
    \lim_{n\to \infty} \frac{\log(1/\epsilon_n)}{\log n} = \frac 12.
  \end{align*}
\end{proposition}
\begin{proof}
  In the proof of \cref{prop:genetic_succ_prob}, we can take $\epsilon_1 = \Theta(\epsilon)$. Then by \cref{lem:branching_to_genetic}, the number of calls in the genetic algorithm (which achieves error probability $x^*-\epsilon$) is $\epsilon_1^{-2 \pm o(1)} = \epsilon^{-2 \pm o(1)}$.
\end{proof}

\subsection{Convergence analysis of the random sampling algorithm}
\begin{proposition} \label{prop:random_sampling_convergence}
  The success probability of the Random Sampling Algorithm (\Cref{alg:random_sampling}) with context size $k$ after $n$ steps is $x^*-\epsilon_n$, where
  \begin{align*}
    \lim_{n\to \infty} \frac{\log (1/\epsilon_n)}{\log n} = \min\left\{ 1/2, 1-k\frac{q-x^*}{1-x^*} \right\}.
  \end{align*}
\end{proposition}
\begin{proof}
  The result essentially follows from \cref{thm:stochastic_approx}.

  Let $h(x) = q - (q - p)(1 - x)^k - x$ as in the proof of \cref{prop:random_sampling_succ_prob}.
  Then $\tau = -h'(x^*) = 1 - k\frac{q - x^*}{1 - x^*}$.
  We have $0 < \tau < 1$ by the proof of \cref{prop:random_sampling_succ_prob}.
  By \cref{thm:stochastic_approx}, the error decays polynomially in $n$ with exponent $-\min\{1/2, \tau\}$.
\end{proof}

\section{Uniform Model} \label{sec:uniform}

In this section, we consider the uniform model. We first recall its definition.

\uniformModel*

We remark that this corresponds to a very simple special case of the decaying model where $f(k)$ is a step function. Clearly, it is best to pass exactly $k$ solutions to the context whenever $k$ solutions are available. We first observe that the uniform model is strongly monotone, and therefore \Cref{lem:monotone-strong} applies.

\begin{observation}\label{obs:uniform_monotone}
  In the uniform model with $q \geq p$, it is always beneficial to add $k$ solutions to the context.
\end{observation}
\begin{proof}
  Follows directly from \Cref{lem:monotone-strong} since the uniform oracle is strongly monotone \Cref{defn:monotone}.
\end{proof}

From this, we immediately obtain that the maximum achievable success probability is the fixed point achieved when always adding $k$ independently generated solutions.

\begin{corollary} \label{corr:uniform_fixed_point}
  The maximum achievable success probability of a $(\calA^{(p, q)}_u, n)$-reasoning algorithm with $q \geq p$ and maximum context length $k$ is the unique solution of
  \begin{align*}
    x = q + (p - q)(1 - x)^k
  \end{align*}
  as $n \to \infty$.
\end{corollary}
\begin{proof}
  Directly follows from \Cref{obs:uniform_monotone} and \Cref{prop:binary_opt} by reformulating.
\end{proof}

The algorithm for achieving this probability may be quite inefficient in terms of the number of oracle calls for achieving a fixed probability. In the following we show that introducing dependencies to reduce the number of oracle calls in fact reduces the achieved success probability. First, we show that this is the case for $k = 2$ as a warm-up in \Cref{clm:sliding_window_subopt_2}. This proof can be seamlessly generalized to larger $k$, which we show in \Cref{prop:sliding_window_subopt}.

\begin{claim}[Warm-up]\label{clm:sliding_window_subopt_2}
  The sliding window approach where the oracle is always called on the $k$ most recently generated solutions achieves sub-optimal performance for $k = 2$.
\end{claim}
\begin{proof}
  We can imagine this process as generating a sequence of solutions $\{s_i\}_{i = 1, \ldots, n}$, and we let $v_i = \score(s_i)$. Then imagine being at some stage $n$. The only quantities that matter for the transfer function describing the distribution of $v_{n + 1}$ are $v_n$ and $v_{n-1}$. We distinguish three states. The state is the smallest number $i$ such that $v_{n - i} = 1$. If no such $i$ exists, or $i \geq k$, the state is $k$. We then obtain the following transition matrix from the definition of the oracle.
  \begin{align*}
  W = \begin{pmatrix}
    q & q & p \\
    1 - q & 0 & 0 \\
    0 & 1 - q & 1 - p
    \end{pmatrix}^\top
  \end{align*}
  We then compute the stationary distribution of this matrix. Solving the system $W^T \pi = \pi$ yields
  \begin{align*}
    \pi = \frac{1}{N}
    \begin{pmatrix}
    p \\
    (1 - q)p \\
    (1 - q)^2
    \end{pmatrix}
  \end{align*}
  where $N = p + (1 - q)p + (1 - q)^2$. We output the correct solution whenever we are in state $0$, so the probability of returning a correct solution is given by $p/N$.

  Now, given \Cref{corr:uniform_fixed_point} we aim to show that $p/N < q + (p - q)(1 - p/N)^2$. Reformulating yields
  \begin{align*}
    (p - q)^2(1 - q)^2 > 0
  \end{align*}
  which is true whenever $q < 1$ and $p \neq q$.
\end{proof}

Next, we generalize the previous claim to all $k$.

\begin{proposition}\label{prop:sliding_window_subopt}
  The sliding window approach where the oracle is always called on the $k$ most recently generated solutions achieves sub-optimal performance.
\end{proposition}
\begin{proof}
  We again consider a sequence of solutions $\{s_i\}_{i = 1, \ldots, n}$ and let $v_i = \score(s_i)$ for $i \in 1, \ldots, n$.

  Then, for $0\le i\le n$, $0\le j\le n$, we let $x_{i,j}$ denote the probability that
  \begin{align*}
    j= \min\{k,\max \{ a: v_i=\cdots=v_{i-a+1}=0\}\}.
  \end{align*}
  Then, we have $x_{0,j}=\mathbbm1_{j=k}$ and $x_{i+1,*}= x_{i,*} W$ where the transition matrix $W$ is given by $W(k,0) = p$, $W(k,k) = 1-p$, $W(i,0) = q$, $W(i,i+1) =1-q$ for $0\le i\le k-1$ and $W$ contains $0$ entries everywhere else. So $x_{n,*}=x_{0,*} W^n$. We remark that this transition matrix is completely analogous to the warm up for $k = 2$ presented in \Cref{clm:sliding_window_subopt_2}.

  The stationary distribution $\pi$ is given by
  $\pi(j) =C (1-q)^j$ for $0\le j\le k-1$, $\pi(k) = C (1-q)^k/p$,
  where $C^{-1} = \sum_{0\le j\le k-1} (1-q)^j + (1-q)^k/p = (1-(1-q)^k)/q+(1-q)^k/p$. Since the algorithm succeeds whenever it is in state $0$, the limiting success probability of the sliding window algorithm is equal to
  $$
    \pi(0) = \left((1-(1-q)^k)/q+(1-q)^k/p\right)^{-1}.
  $$
  In general, this is not a solution to the equation $x=q+(p-q)(1-x)^k$.
  So the sliding window algorithm is not optimal for general $k$.
\end{proof}

Therefore, even in this extremely simple model correlations between the solutions are an issue, which motivates considering the branching and geometric algorithm. We remark that for certain decaying models (such that the exponential decaying model and the polynomial decaying model) the success probability of a sliding window algorithm goes to $0$ as the number of steps goes to infinity because the context will eventually only contain wrong solutions. After this point, the oracle never generates a correct solution again. This can be seen as modeling the ``overthinking'' phenomenon observed in practice \cite{zhang2025more}.

\section{Exponentially and Polynomially Decaying Models} \label{sec:decay}

Next, we analyze the exponentially and polynomially decaying models $\calA^{(p, f)}_d$, which are a more realistic model that introduces a non-trivial dependency on the context size (See \Cref{sec:experiments} for experimental results). We first recall the general model definition and then the exponentially and polynomially decaying special cases. 

\decayingModel*

\expDec*

\polyDec*

Unlike the uniform setting, where it is always beneficial to pass as many solutions as possible in the context, a non-constant decay function $f(k)$ realizes a non-trivial trade off. In the following, we explore the two decay functions that are arguably most natural, exponential and polynomial decay (See \Cref{def:exp_decay} and \Cref{def:poly_decay}). We sometimes write $\mathcal{A}^{(f, p)}$ when referring to the exponential and polynomial decay model with $g(k) = p\mathbbm{1}_{\{k = 0\}}$.

\subsection{Exponential Decay}

In this section, we examine exponential decay functions $e_q(k) = q^{k - 1}$ (See \Cref{def:exp_decay}). These already lead to a non-trivial trade off when selecting the context size. Due to the quick reduction in success probability, they encourage only combining a small amount of solutions.

\paragraph{A two stage process. } We first show how this trade-off is realized by examining the simple two stage process $pass@k$. In the first stage, we assume to be given an algorithm that samples an arbitrarily sized collection of solutions, each of which is correct independently with some probability $x$. Then, we aim to select an optimal number of such solutions to put into the context for a single oracle call to maximize the probability of correctness of the generated solution. This corresponds to maximizing the success probability that a $pass@k$ algorithm can obtain.

\begin{claim} \label{clm:optimal_k}
    The optimal amount of independent solutions that are correct with probability $x$ to pass to the oracle $\calA^{(e_q, p)}_d$ for $x \geq p$ is
    \begin{align*}
        \arg \max_{k \in \N_{\geq 0}} q^{k - 1}(1 - (1 - x)^k) = \max\left(1, \floor{1 + \frac{\log\left(\frac{1 - q}{1 - (1- x)q}\right)}{\log(1 - x)}}\right).
    \end{align*}
\end{claim}
\begin{proof}
    Since passing a single solution yields success probability $x \geq p$ it is never beneficial to pass $0$ solutions. We now analyze the function $g(k) \defeq q^{k - 1}(1 - (1 - x)^k)$ that we aim to maximize. We examine how this function changes as we increase $k$.

    We have $f(k + 1)/g(k) = q\frac{1 - (1 - x)^{k + 1}}{1 - (1 - x)^k}$. Since the term $\frac{1 - (1 - x)^{k + 1}}{1 - (1 - x)^k}$ is monotonically decreasing with $k$, we have to find the largest $k$ for which $q\frac{1 - (1 - x)^{k + 1}}{1 - (1 - x)^k} \geq 1$ which is equivalent to
    \begin{align} \label{ineq:growth}
        (1 - x)^k \geq \frac{1 - q}{1 - q(1 - x)}.
    \end{align}
    We then solve for the continuous $k'$ for which $\eqref{ineq:growth}$ holds with equality by taking the logarithm on each side. We obtain that $(1 - x)^{k'} = \frac{1 - q}{1 - q(1 - x)}$ for
    \begin{align*}
        k' = \frac{\log\left(\frac{1 - q}{1 - (1- x)q}\right)}{\log(1 - x)}.
    \end{align*}
    Since this shows there is still an increase when going from $k'$ to $k' + 1$ in this case, we obtain that the optimal discrete solution is as claimed.
\end{proof}

We observe that whenever progress is possible for some $k$, setting $k = 2$ also makes progress. This property of the exponential decay model enables us to derive a closed form solution of the maximum achievable success probability.

\begin{corollary}\label{corr:two_makes_progress}
    Whenever $\arg \max_{k \in \N_{\geq 0}} q^{k - 1}(1 - (1 - x)^k) > k$, then $q(1 - x)^2 \geq x$.
\end{corollary}
\begin{proof}
    As observed in the proof of \cref{clm:optimal_k}, the function $g(k) = q^{k - 1}(1 - x)^k$ first monotonically increases until it reaches the optimal step $k$, and then monotonically decreases. Since $g(1) = x$, we are guaranteed that $g(2) > x$ whenever the optimal $k$ is at least $2$. This proves the corollary.
\end{proof}

\paragraph{The maximum achievable success probability. }

Similar to the case of the uniform model, we study the maximum achievable success probability with the exponential decay function.

\begin{lemma}
    The maximum achievable success probability of a $(\calA^{(e_q, p)}_d, n)$-reasoning algorithm with $q \geq p$ is
    \begin{align*}
        \max\left(p, 2 - \frac{1}{q}\right).
    \end{align*}
\end{lemma}
\begin{proof}
    We start with the same observations as in the proof of \cref{sec:succ_prob}, namely that all solutions generated by a reasoning algorithm are either positively correlated or independent, and that it is always optimal to pass independent solutions with as high correctness probability as possible. We therefore obtain from \Cref{prop:binary_opt} that the solution has to be a fixed point of the function $f(x) = q(1 - (1 - x)^2)$ because the largest fixed point is achieved at $k = 2$ by \Cref{corr:two_makes_progress}. This  function has only one fixed point $x = 2 - \frac{1}{q}$.
\end{proof}

We remark that it is never useful to pass more than one solution as context to $\calA_d^{(e_q, p)}$ when $q \leq \frac{1}{2}$, as in this case the model performs worse than randomly sampling a solution from the context. Therefore, we have $q \geq 1/2$ for all boostable models.

\paragraph{Convergence analysis. }

In the following, we analyze the speed of convergence. In contrast to \Cref{sec:convergence}, we focus on the behavior when the initial success probability approaches $0$. This corresponds to the setting where boosting is the most useful. 

Here, we assume $k = 2$ throughout to achieve a lower bound. Furthermore, we will artificially weaken the oracle by a factor $(1 - \epsilon)$. This allows us to immediately apply \Cref{lem:branching_to_genetic} to obtain a genetic algorithm and will not significantly change the convergence analysis.

We consider the branching algorithm that starts by generating a first generation of solutions with success probability $p_0 = p$, and then iteratively generates the next generation of solutions with success probability $p_i$ by combining two fresh solutions from the previous generation. Formally, this is equivalent to the branching algorithm (See \Cref{alg:branching}) with $k_i = 2$ for all $i$.

\begin{lemma}\label{lem:conv_exp}
    Let $p_{\max} = 2 - \frac{1}{q} > p$. Then we obtain a solution with success probability $p_{\max} - \epsilon$ after $L = O(\frac{1}{\epsilon}\log(1/p))$ generations. This also holds if the success probability of the oracle is decreased by a multiplicative factor $1 - \epsilon/8$.
\end{lemma}
\begin{proof}
    We have $p_{i + 1} = q(1 - (1 - p_{i})^2)$. We will show that $p_{i + 1} \geq p_i (1 + \frac{\epsilon}{4})$, or equivalently
    \begin{align*}
        p_i (1 + \frac{\epsilon}{4}) \leq (1 - \epsilon/8) \cdot q(1 - (1 - p_i)^2)
    \end{align*}
    which is implied by
    \begin{align}\label{eq:gain}
        p_i(1 + \frac{\epsilon}{2}) \leq q (1 - (1 - p_i)^2)
    \end{align}
    as long as $p_i \leq p_{\max} - \epsilon$. Reformulating \eqref{eq:gain} yields
    \begin{align*}
        p_i \leq 2 - \frac{1 + \frac{\epsilon}{2}}{q}.
    \end{align*}
    Since $q \in [\frac{1}{2}, 1]$, this condition always holds when $p_i \leq p_{\max} - \epsilon$. Therefore, the success probability increases by at least a multiplicative factor $(1 + \frac{\epsilon}{4})$ in each iteration.
\end{proof}

\begin{corollary} \label{corr:conv_exp}
    There is a genetic algorithm (See \Cref{alg:genetic}) with $O(\log(1/\epsilon)\log(1/p)/(p \epsilon^3))$ oracle calls that achieves success probability $p_{\max} - \epsilon$. The sequential depth of the algorithm is $O(\log(1/p)/\epsilon)$.
\end{corollary}
\begin{proof}
    Directly follows from \Cref{lem:conv_exp} and \Cref{lem:branching_to_genetic}.
\end{proof}

In our proof, we merely focus on context sizes equal to $2$. For very large $q$, there can however be substantial gains in the number of rounds when combining more answers at a time. In particular, for $q = (1 - p)$ we can get rid of the dependence on $\log(1/p)$ in the number of rounds entirely by initially combining $\approx 1/p$ solutions. For most practical settings where $q$ should be thought of as a constant close to $1$, the studied setting captures the behavior.

\subsection{Polynomial Decay}

In this section, we examine polynomial decay functions (See \Cref{def:poly_decay}). Since these functions drop off much slower than the exponential decay we examined in the previous section, they allow much faster boosting. In particular, we can show that approximately $\approx \log \log 1/p$ iterations suffice to boost to constant success probability with a simple algorithm.

\paragraph{A simple algorithm. }

We describe a simple algorithm that automatically scales $k$ and quickly increases the success probability to a small constant. Given access to independent solutions with correctness probability at least $p$, the algorithm samples $k = \floor{1/p}$ such solutions as context $C$ and passes them to the oracle $\calA^{(p_q, p)}_d(C)$. Then, the oracle generates a solution that is correct with probability $p'$ bounded by
\begin{align*}
    p' \geq \frac{1}{k^q}(1 - (1 - p)^k) \geq p^q(1 - (1 - p)^{1/p - 1}).
\end{align*}
For $p \leq 1/2$, we obtain
\begin{align*}
    p' \geq p^q(1 - 2/e).
\end{align*}
For a constant $q < 1$, this series will recover a squaring type behavior. We next show that this series increases very quickly for a general constant $c$. This later allows us to replace $(1 - 2/e)$ with $(1 - 2/e)(1 - \epsilon')$ to weaken the oracle and obtain convergence guarantees for the genetic algorithm via \Cref{lem:branching_to_genetic}.

\begin{claim}\label{clm:series_poly}
    Consider $p_0 = p$, and $p_i = p_{i - 1}^q \cdot c$ for $i \geq 1$. Then, $p_n = p^{q^n} \cdot c^{\frac{1 - q^n}{1 - q}}$ for $n \geq 1$.
\end{claim}
\begin{proof}
    The proof is by induction. The base case $n = 1$ directly follows from the definition of $p_1$.

    Now, we assume the recurrence holds for $n$ and aim to show it for $n + 1$. We have
    \begin{align*}
        p_{n + 1} &= p_n^q \cdot c \\
                  &= \left(p^{q^n} \cdot c^{\frac{1 - q^n}{1 - q}}\right)^q \cdot c \\
                  &= p^{q^{n + 1}} \cdot c^{1 + q\frac{1 - q^n}{1 -q}} = p^{q^{n  + 1}}c^{\frac{1 - q^{n + 1}}{1 - q}}.
    \end{align*}
\end{proof}

From \Cref{clm:series_poly}, we observe that this simple algorithm converges to a success probability of $c^{\frac{1}{1 - q}}$. For constant $q$ and $c$, this corresponds to a constant success probability.

\begin{claim}\label{clm:convergence_poly}
    We have $p_n \geq (1 - \epsilon)c^{\frac{1}{1 - q}}$ for some $n = O(\log \log(1/p)/q + \log(\epsilon^{-1})/q)$.
\end{claim}
\begin{proof}
    We have $p_n = p^{q^n} c^{\frac{1 - q^n}{1-q}} = p^{q^n}c^{\frac{1}{1-q}}c^{\frac{-q^n}{1-q}}$.
    We first show that $c^{\frac{-q^n}{1-q}} \geq (1 - \epsilon/2)$ for $n = O(\log \log(1/p) + \log(\epsilon^{-1}/q)$. We take the logarithm on both sides
    \begin{align*}
        \frac{-q^n}{1-q} \log(c) \leq \log(1 - \epsilon/2) \\
        q^n \geq - (1 - q) \frac{\log(1 - \epsilon/2)}{\log(c)}
    \end{align*}
    We then approximate $\log(1 - \epsilon) \approx \epsilon$ and collect the constant terms into $C$ obtaining the condition
    \begin{align*}
        n \geq C' \log(1/\epsilon)/q
    \end{align*}
    for some constant $C'$.

    Next, we show that $p^{q^n} \geq (1 -\epsilon/2)$. We have
    \begin{align*}
        p^{q^{n}} &\geq (1 - \epsilon/2) \\
        q^n \log(p) &\geq \log(1 - \epsilon/2) \\
        n &\geq C (\log(\epsilon^{-1}) + \log \log(1/p))/q
    \end{align*}
    for some large constant $C$. This concludes the proof.
\end{proof}

From \Cref{clm:convergence_poly}, we obtain that this algorithm indeed converges very quickly. Furthermore, it does so even when $c$ is slightly smaller. This allows it to pair very well with the genetic algorithm described in \Cref{alg:genetic}. We therefore derive a genetic algorithm. 

In the following, we focus on the behavior when the initial success probability approaches $0$. Once a reasonably high (e.g. constant) success probability is achieved, the type of convergence analysis from \Cref{sec:convergence} could be applied to obtain the speed of convergence for the final bit. 

\begin{lemma} \label{lem:conv_poly}
    There is a genetic algorithm that using the model $\calA^{(p_q, p)}_d$ finds a solution that is correct with probability at least $(1/16)^{1/(1 - q)}$ in $O(\frac{1}{pq} \log \log 1/p)$ oracle calls.
\end{lemma}
\begin{proof}
    Directly follows from \Cref{clm:convergence_poly} and \Cref{lem:branching_to_genetic} and $(1 - (1 - p)^{1/p - 1} \geq (1 - 2/e)$ for $p \geq 1/2$.
\end{proof}

\paragraph{The maximum achievable success probability. } To characterize the maximum achievable success probability we resort to \Cref{prop:binary_opt} since there does not appear to be an obvious simpler characterization.

\section{Conclusion} \label{sec:conclusion}

We introduce a theoretical model that formalizes reasoning oracles that are given an array of solutions as context. These formalize the experimentally observed behavior that adding correct solutions to the context helps, but the accuracy decreases with the context size. The derived theory provides a framework for developing and analyzing reasoning algorithms that have been shown to perform well in practice \cite{Ven25}. 

\paragraph{Future work. } The algorithms we develop are naturally parallel. It is interesting to further explore the trade off between solution quality, parallelism and the total number of model calls. We believe that less parallelism could lead to savings in logarithmic factors that are highly relevant given the cost of executing such algorithms in practice. 

\paragraph{Limitations and future model extensions. } We limit ourselves to the binary model, where a solution is either correct or not. While this enables us to devise simple experiments, this is a significant drawback when considering more difficult questions for which we cannot hope to obtain correct solutions directly. 

It therefore makes sense to extend our models to more fine grained score functions in the future. Such an extension should go hand in hand with the introduction of a measure of answer diversity. That is, because combining two partially correct solutions should only be beneficial if they complement each other. This could be modeled by additionally associating a unit vector with each solution where orthogonal vectors correspond to solutions that perfectly complement each other. These vectors can either be hidden alongside the score function or exposed through some oracle, where the former is simpler than the latter but yields less flexibility when developing algorithm. 

While our model offers exciting opportunities for extensions, we caution that it is important to ground them in experiments to ensure that the theory captures something interesting about the world we live in.

\newpage

\bibliographystyle{alpha}
\bibliography{refs}

\newpage

\appendix
\crefalias{section}{appendix}

\section{Experiments}\label{sec:experiments}

In this section, we verify some of our theoretical assumptions with the AIME 2025 math dataset and the Gemini 2.5 Pro model. The dataset became available after the model, which helps us avoid test set contamination. AIME 2025 has 30 questions, and we plot the average accuracy for each question across independent model calls in Figure~\ref{fig:overall_aime_acc}. The prompt consists only of the question. As expected, Gemini 2.5 Pro has a strong performance on many of the questions.

\begin{figure}[ht!]
    \centering

    \vspace{5mm} 
    
    \includegraphics[width=0.88\textwidth]{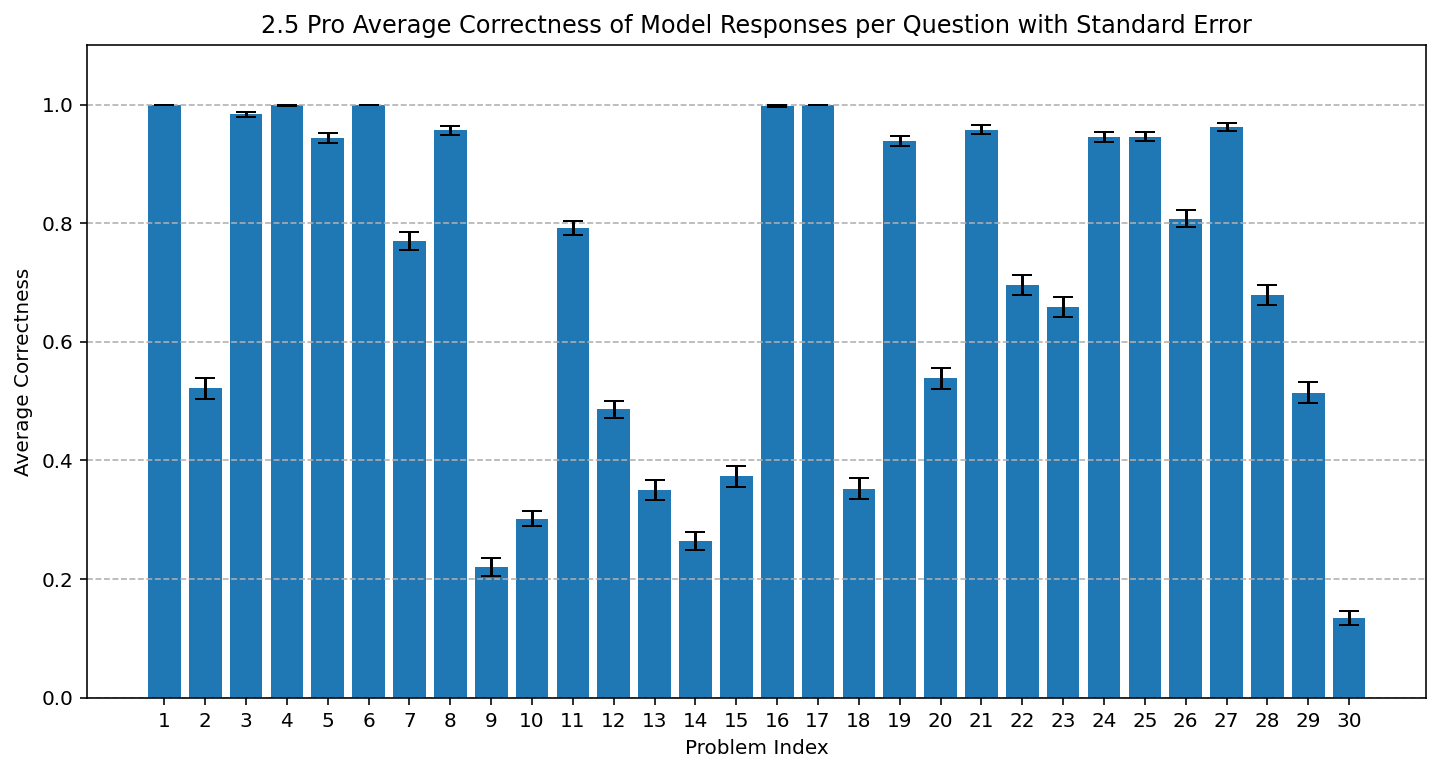}
    
    \caption{Average AIME 2025 accuracy per question for Gemini 2.5 Pro. We use 780 model calls per question. The error bars represent standard error.}
    \label{fig:overall_aime_acc}
\end{figure}

Next, we saved the outputs and picked questions 10 and 12 to examine model accuracy when solutions are provided in context. These are questions for which the model struggled but occasionally generated correct answers. We study two settings below.

\subsection{Fixing one correct solution and increasing incorrect solutions}
We start by investigating model accuracy when only one of the solutions is correct, and the number of incorrect solutions grows from $0$ to $12$. For each of the configurations, we make $30$ calls with Gemini 2.5 Pro and make sure none of the solutions are repeated across any of the calls. The solutions are also shuffled before placing them in context. Our results are in Figure~\ref{fig:pro_one_correct}, and we observe a decaying accuracy as the number of incorrect questions increases. This motivates our theoretical models with decay.

\begin{figure}[ht!]
    \centering
    
    \includegraphics[width=0.62\textwidth]{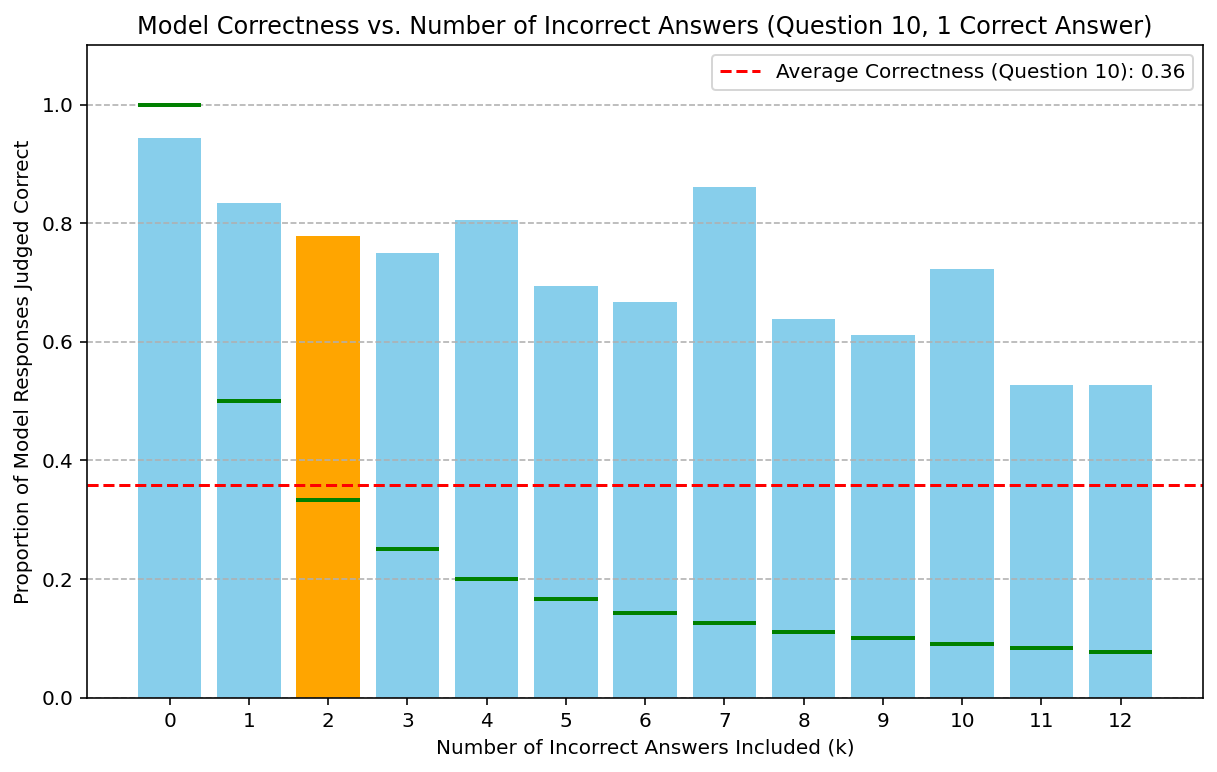} 
        
    \includegraphics[width=0.62\textwidth]{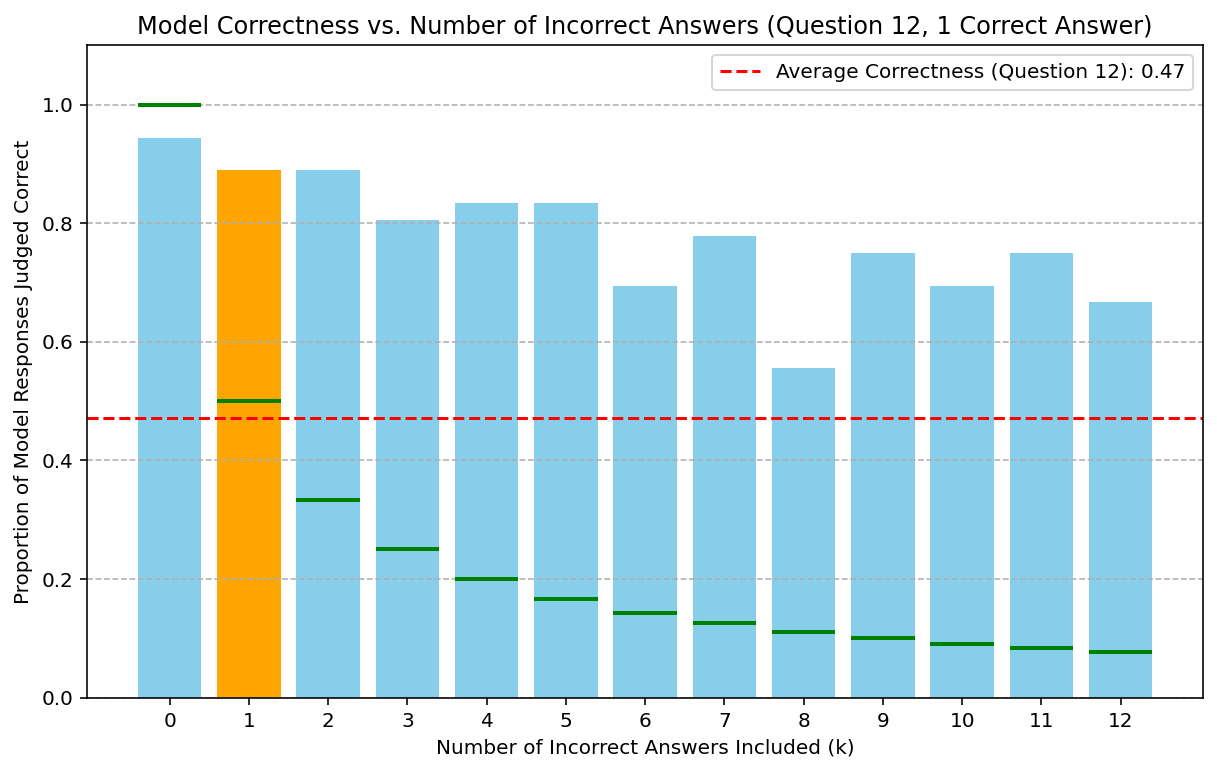}
    
    \caption{Gemini 2.5 Pro accuracy when providing $1$ correct solution answer and $k=0$ to $12$ incorrect solutions. The red line represents base accuracy of Pro on the question without any solutions. The green lines represent $1/(k+1)$, the accuracy when returning a solution in the context uniformly at random. The orange bar is the most likely configuration according to the base accuracy (the $k$ for which $1/(k+1)$ is closest to base accuracy). }
    \label{fig:pro_one_correct}
\end{figure}

\subsection{Fixing total number of solutions and varying correct solutions}

Finally, we move to a setting where instead of fixing one correct answer, we fix the total number of solutions and vary the number of correct versus incorrect answers.  We show results in Figure~\ref{fig:pro_five_total}. We again used $30$ calls per configuration, did not reuse solutions across any model calls, and shuffled the order of the solutions. 

We observe a very smoothly increasing accuracy as the number of correct solutions increases. We also note that the performance of this setup---sampling $5$ solution answers and providing them as context for a final model call---always improves over the base accuracy of the model on average. On the two questions we examined for Pro, the gap is about $40\%$, suggesting that it is a particularly strong verifier.

\begin{figure}[ht!]
    \centering
    
    \includegraphics[width=0.7\textwidth]{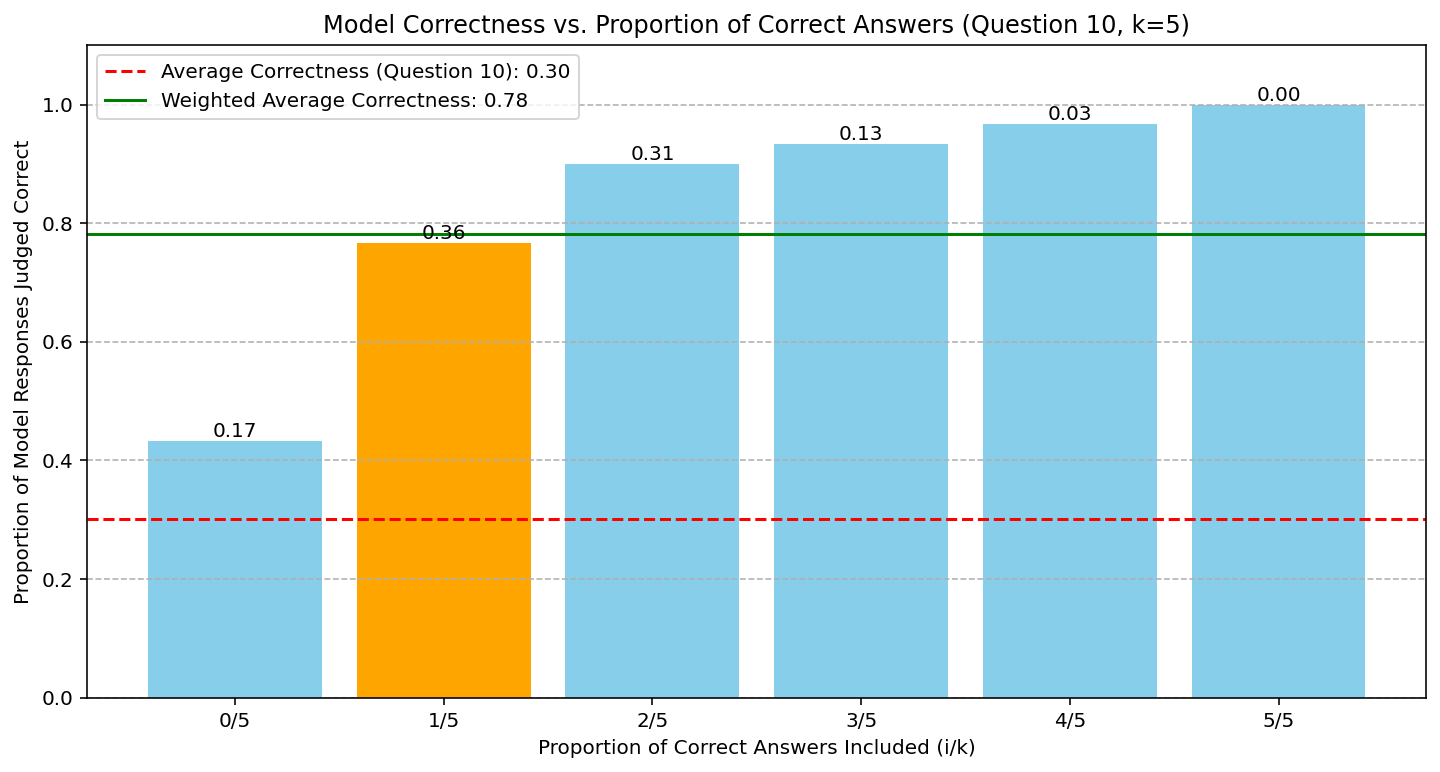}    
    
    \includegraphics[width=0.7\textwidth]{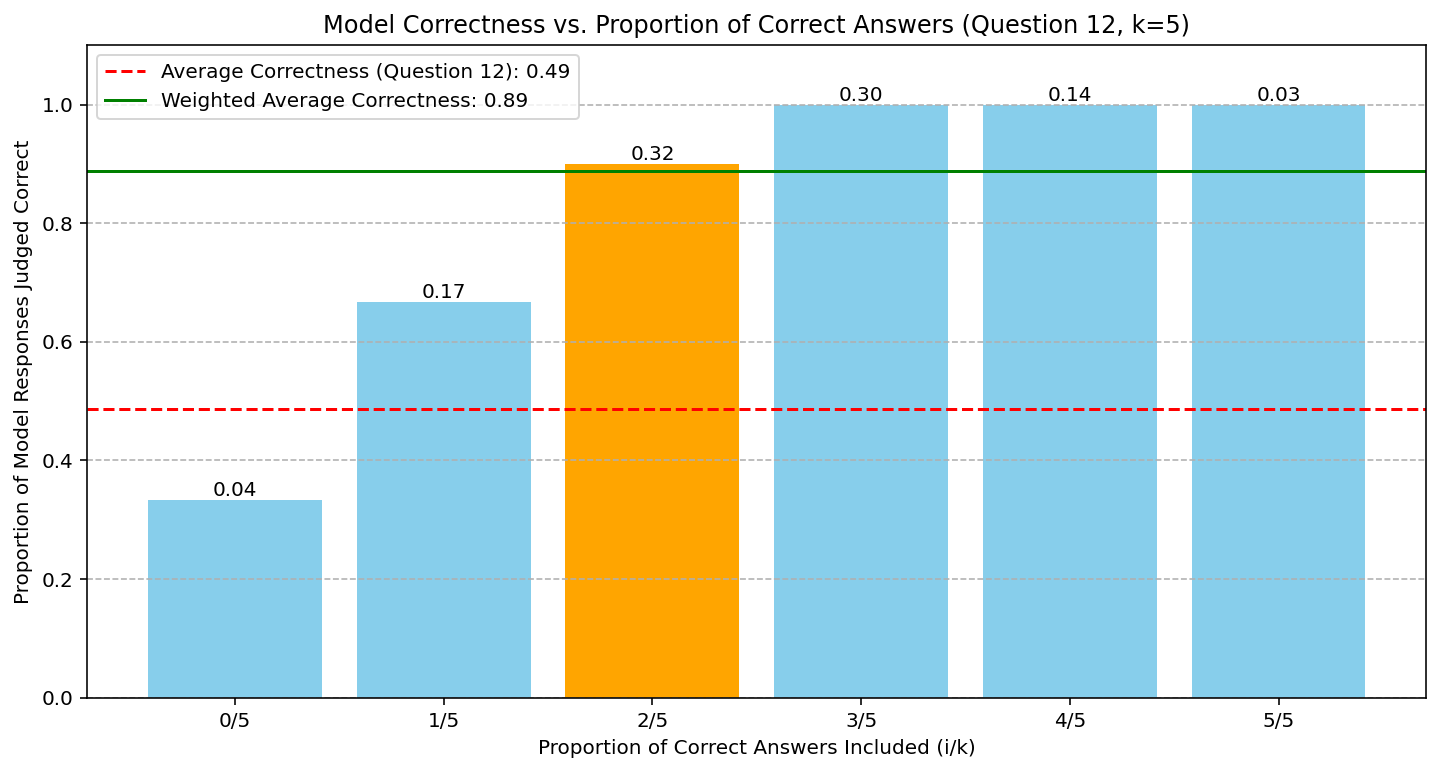}
    
    \caption{Gemini 2.5 Pro accuracy when providing $5$ total solutions and varying the number of correct versus incorrect answers. The red line represents base accuracy of Pro on the question without any solutions. The numbers on top of each bar represent the probability of that configuration according to the average correctness. The orange bar highlights the most likely configuration. The green line uses these weights to give the accuracy when sampling $5$ solutions and using them as context for a final model call.  }
    \label{fig:pro_five_total}
\end{figure}

\end{document}